\documentclass{article}

\PassOptionsToPackage{sort,numbers}{natbib}
\usepackage[final]{nips_2018}

\usepackage{amsmath, amssymb, amsthm}
\usepackage{graphicx}
\usepackage{times}
\usepackage{subfigure} 
\usepackage{array}
\usepackage{caption} 
\usepackage{tabularx}
\usepackage{cellspace}
\newcolumntype{Y}{>{\tiny\raggedleft\arraybackslash}X}
\usepackage{algorithm, algorithmic}
\usepackage{color}\definecolor{darkblue}{rgb}{0,0,0.75}
\usepackage[pdftex,pdftitle={Corresponding Projections for Orphan Screening},colorlinks=true, bookmarksnumbered, citecolor=darkblue, urlcolor=darkblue, linkcolor=darkblue, 
 pdfpagemode=UseNone]{hyperref}
\usepackage[utf8]{inputenc}
\usepackage{xfrac}
\usepackage{wrapfig}
\usepackage{authblk}

\newtheorem{thm}{Theorem}

\newtheorem{lm}[thm]{Lemma}

\newtheorem{defin}[thm]{Definition}

\newcommand{\defemph}[1]{\emph{#1}}

\newcommand{\ueq}[1][]{%
	\if\relax\detokenize{#1}\relax
	\sbox0{$\underbrace{=}_{}$}%
	\mathrel{\mathmakebox[\wd0]{=}}
	\else
	\mathrel{\underbrace{=}_{\mathclap{#1}}}
	\fi}

\title{Corresponding Projections for Orphan Screening}

\author[1,2]{Sven Giesselbach}
\author[3]{Katrin Ullrich}
\author[2,3]{Michael Kamp}
\author[1,2]{Daniel Paurat}
\author[4]{Thomas G\"artner}

\affil[1]{Fraunhofer IAIS}
\affil[2]{Competence Center Machine Learning Rhine-Ruhr}
\affil[3]{University of Bonn}
\affil[4]{University of Nottingham}
\affil[ ]{\texttt{\{sven.giesselbach,daniel.paurat\}@iais.fraunhofer.de}}
\affil[ ]{\texttt{\{ullrich@iai,kamp@cs\}.uni-bonn.de}}
\affil[ ]{\texttt{thomas.gaertner@notthingham.ac.uk}}

\begin{document}

\maketitle
\begin{abstract}
We propose a novel transfer learning approach for orphan screening called corresponding projections. 
In orphan screening the learning task is to predict the binding affinities of compounds to an orphan protein, i.e., one for which no training data is available. 
The identification of compounds with high affinity is a central concern in medicine since it can be used for drug discovery and design.
Given a set of prediction models for proteins with labelled training data and a similarity between the proteins, corresponding projections constructs a model for the orphan protein from them such that the similarity between models resembles the one between proteins.
Under the assumption that the similarity resemblance holds, we derive an efficient algorithm for kernel methods. We empirically show that the approach outperforms the state-of-the-art in orphan screening. 
\end{abstract}

\section{Introduction}
\label{sec:intro}

This paper proposes an approach to predicting binding affinities of compounds to a protein without training data.
In biological organisms the bindings of small compounds to proteins induce subsequent cellular reactions. The strength of this binding is expressed via a real-valued \textit{affinity}. 
In the context of affinity prediction, we refer to compounds as \defemph{ligands}. 
Protein-ligand-complexes regulate a variety of biochemical processes, e.g., the effectiveness of transporters, ion channels, hormones, receptors, and enzymes. To know whether or how strong a ligand binds to a protein is crucial for \textit{drug discovery} and \textit{design}. The identification of ligands with high affinity for new drug substances is a central concern in medicine~ \citep{JacHoffStoVert08}. Via high-throughput screening (HTS) machinery in laboratories one is able to determine ligand affinities practically. In order to supplement this time-consuming and cost-intensive procedure, molecular databases can be screened with computational methods~\cite{AinAleksRoessBall15, LiWangMer11, ZhouSkol12, GepVogtBaj10}, denoted \textit{in-silico virtual screening}~\citep{Shoichet04}.

Ligands can be represented by \textit{molecular fingerprints}~\citep{BenJenScheiSukGliDav09}, i.e., vectorial representations that comprise their structural or physico-chemical information. For a protein with a training set of ligands and their binding affinities, this allows to train a prediction model using similarity search \cite{GepHumStuGaeBaj09} and various machine learning approaches, e.g., random forests or neural networks \cite{JimSkaMartRosFab18, BockGough05, LoRenTorAlt18, KunPaulBan18}. The most prominent and successful methods so far are \textit{support vector machines (SVM)} using molecular fingerprints \cite{GepHumStuGaeBaj09, Sugaya14, MaunzHelma08, UllMaWel16,ullrich2017co}.

This paper tackles an even more challenging task called \textit{orphan screening}~\citep{WassGepBaj09}. It describes the prediction of ligand affinities for proteins without known ligand affinity values (\textit{orphan targets}), such as 
\begin{wrapfigure}{r}{0.41\textwidth}
	\begin{center}
		\includegraphics[width=0.41\textwidth]{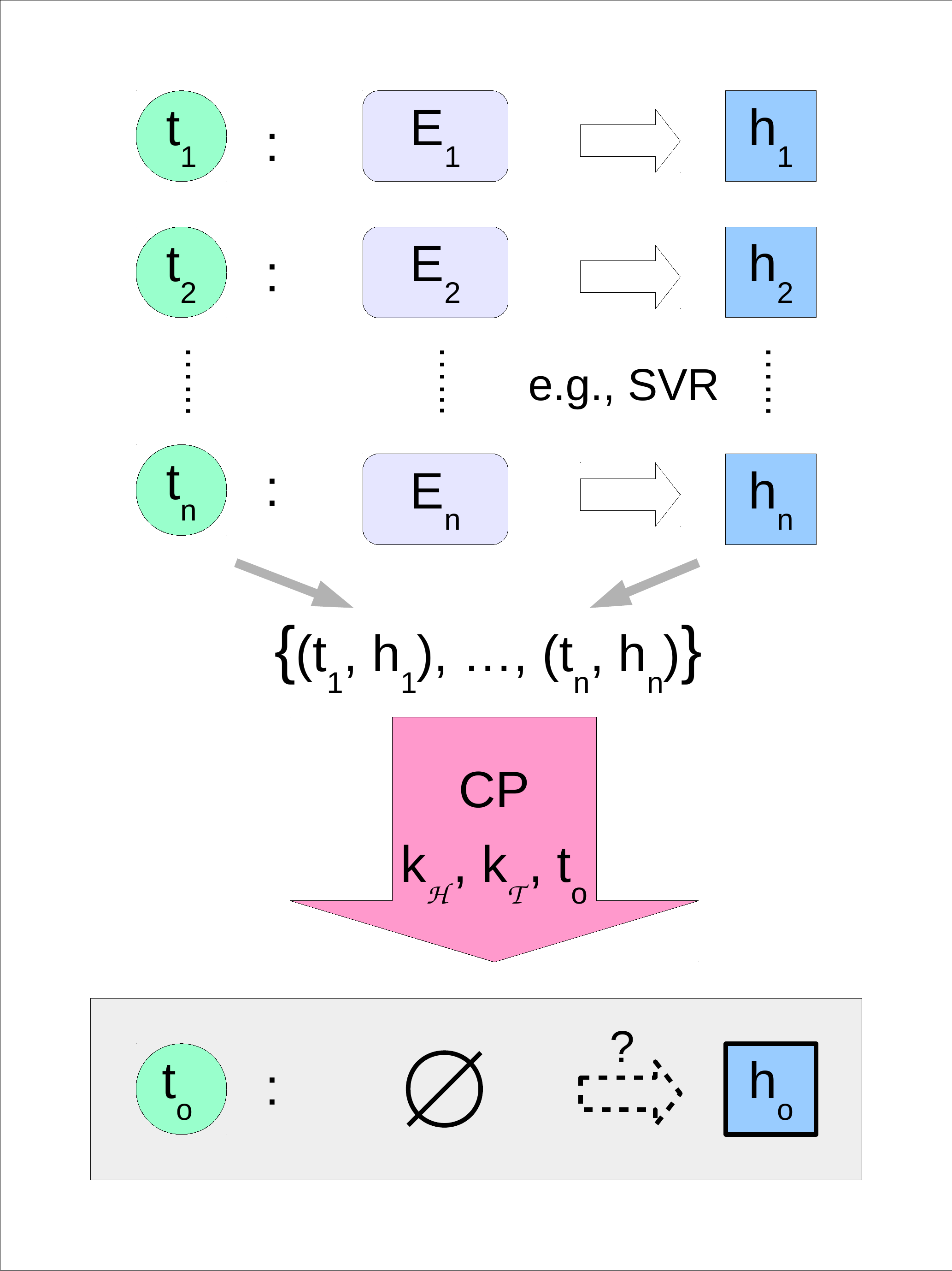}
	\end{center}
	\caption{Overview of the corresponding projections approach (CP). The goal is to find a hypothesis $h_o$ for the orphan target $t_o$ using a set of supervised targets $t_1,\dots,t_n$ with labelled training sets $E_1,\dots,E_n$. For each $t_i$ with $i\in [n]$, a model $h_i$ is trained using $E_i$, e.g., by SVR. The CP algorithm uses these proteins and models, as well as a similarity measure $k_\mathcal{T}$ between targets and a similarity measure $k_\mathcal{H}$ between models to construct $h_o$.}
	\label{overview_cp}
\vspace{-20px}
\end{wrapfigure}
 \textit{G-protein coupled receptors} (GPCRs) which are popular drug targets with few or no identified ligands \cite{JacHoffStoVert08, ErhLHeureux06}. Note that in previous work on orphan screening (see, e.g., \citep{ErhLHeureux06, JacHoffStoVert08, NingRangKar09, WassGepBaj09, GepHumStuGaeBaj09, BockGough05}), only those compounds with an affinity above a predifined threshold are called ligands and the task is to classifiy a compound as ligand or non-ligand. In the present work, we consider the regression case, i.e., the direct prediction of the binding affinities. Thus, for simplicity we refer to every compound as ligand.


The state-of-the art in orphan screening is the application of support vector machines with \textit{target-ligand-kernels} (TLK)~\citep{JacobVert08}.
The TLK is a tensor product of a target kernel and a ligand kernel. Thus, it serves as similarity measure for target-ligand-pairs and allows for a simultaneous screening of proteins and ligands in \textit{chemogenomics}~\citep{JacHoffStoVert08}. 

We present a novel transfer learning approach~\citep{PanYang10} called \textit{corresponding projections} (CP) for orphan screening. The idea behind CP is to relate the projections of proteins to the projections of the corresponding prediction models. Their relationship is used to generate a model for the orphan target.
Just like the TLK approach, the CP requires that labelled training information for other proteins (\textit{supervised targets}) is available for which a prediction model can be trained.
Subsequently, in the actual CP optimisation step a prediction model is assigned to the orphan target such that its relationship to the given models resembles the relationship of the orphan protein to the given ones (see Fig.~\ref{overview_cp}). 
The following section defines CP for affinity prediction using \textit{support vector regression} (SVR) as supervised training method.

In Sec.~\ref{sec:experiments} the CP approach is empirically evaluated on an orphan screening task and compared to TLK and other baselines, including a simplified variant of CP, i.e., a target similarity-weighted sum of supervised models introduced in~\citep{GepHumStuGaeBaj09}. 
Sec.~\ref{sec:conclusion} concludes the paper. The solutions of CP variants with proof and further practical result can be found in the appendix.


\section{Corresponding Projections}
Given a set of $n\in\mathbb{N}$ target proteins $t_1,\dots,t_n\in\mathcal{T}$ with corresponding training sets $E_1,\dots,E_n$, each consisting of labelled examples 
\[E_i=\{(x_1,y_1),\dots,(x_m,y_m)\}\subset\mathcal{X}\times\mathcal{Y}
\]
with $i\in [n]$, a model ${h_i : \mathcal{X}\mapsto\mathcal{Y} \in\mathcal{H}}$ can be trained for each target, e.g., via SVR. Thus, we obtain pairs $(t_1, h_1), \ldots, (t_n, h_n) \in \mathcal{T} \times \mathcal{H}$. Furthermore, let $t_o\in\mathcal{T}$ denote the orphan protein for which we want to infer a model $h_o\in\mathcal{H}$.

For that, let $\langle,\rangle_{\mathcal{T}}$ and $\langle,\rangle_{\mathcal{H}}$ be inner products with associated norms for targets $t \in \mathcal{T}$ and models $h \in \mathcal{H}$. 
We assume that the projection of proteins resembles the projection of models, i.e.,
\begin{align}
  \frac{\langle t_i,t_o \rangle_{\mathcal{T}}}{\|t_i\|_{\mathcal{T}}} \approx \frac{\langle f(t_i),f(t_o) \rangle_{\mathcal{H}}}{\| f(t_i) \|_{\mathcal{H}}} \label{sims_equality}
\end{align}
for every supervised target $t_i$. With this, finding $h_o$ can be solved by a least squares approach.
\begin{defin}\label{def_cp}
The model for the orphan protein $h_o$ is given by $h_o = \operatorname{argmin}_{h \in \mathcal{H}} \mathcal{Q}_o(h)$, where
\begin{align}
\mathcal{Q}_o(h) =\nu \|h\|^2_{\mathcal{H}} + \sum\limits_{i=1}^n \left| \langle h, h_i \rangle_{\mathcal{H}} \|t_i\|_{\mathcal{T}} - \langle t_o, t_i \rangle_{\mathcal{T}} \|h_i\|_{\mathcal{H}}\right|^2, \label{def_corr_proj}
\end{align}
with trade-off parameter $\nu \geq 0$. The optimisation in (\ref{def_corr_proj}) is called \textit{corresponding projections} (CP).
\end{defin}

If $\mathcal{H}$ is a Hilbert space, the model $h_o$ lies in the span of the supervised hypotheses $h_1, \ldots, h_n$, i.e.,
$h_o = \sum_{i=1}^n \beta_{oi} h_i$ for appropriate $\beta_o \in \mathbb{R}^n$. This parametrisation of $h_o$ leads us to a solution of CP.
\begin{lm} \label{lemma_solution_cp}
Let $\mathcal{H}$ be a Hilbert space and $k_{\mathcal{T}}$ a similarity measure on the target space $\mathcal{T}$. CP in (\ref{def_corr_proj}) can be solved via
\begin{align}
\beta_o = [\nu G + GNG]^\dagger G \rho_o, \nonumber
\end{align}
where $[M]^\dagger$ denotes the \textit{Moore-Penrose} inverse of matrix $M$ and 
\begin{equation*}
\begin{split}
N =& \operatorname{diag}(\{k_{\mathcal{T}}(t_i,t_i)\}_{i=1}^n)\\
G =& \{\langle h_i, h_j \rangle_{\mathcal{H}} \}_{i,j=1}^n\\
\rho_o =& \{\sqrt{k_{\mathcal{T}}(t_i,t_i)} ~k_{\mathcal{T}}(t_o,t_i)\|h_i\|_{\mathcal{H}} \}_{i=1}^n\enspace .
\end{split}
\end{equation*}
\end{lm}
The proof of Lemma \ref{lemma_solution_cp} and more theoretical details can be found in Appendix~\ref{app:theory}.

\section{Empirical Evaluation}
\label{sec:experiments}
We evaluate the proposed CP approach on an orphan screening task and compare it to state-of-the art baselines. 
For that, we use $9$ protein-ligand datasets extracted from BindingDB~(\url{bindingdb.org}). Each dataset corresponds to a human protein with peptidase domain and comprises between $240$ and $2649$ ligands with affinity labels (p$K_i$-values) towards the respective protein. We utilise the standard molecular fingerprint ECFP4 \cite{RogHahn10} for the representation of ligands. Target similarities were calculated from the amino acid sequence similarity of the peptidase domain. We normalise the similarities of all targets except for the respective orphan to sum to $1$.

The parameter $\nu$ of the optimisation problem (Eq.~\ref{solution_cp1}) was determined on an independent training set and fixed to $\nu=5$ for all orphan targets. To improve the numerical stability of inversion of $\nu G + GNG$, we add a constant $\lambda\in\mathbb{R}_+$ to the diagonal of the matrix. That is, we consider a slightly different CP optimisation
\[
\beta_o = [\nu G + \lambda \mathbf{I}_n + GNG]^\dagger G \rho_o\enspace ,
\]
with $\lambda=1$ for all orphan targets~\footnote{The code can be found at \url{bitbucket.org/grumpy_kat/corresponding-projections}.}.
%
%

In order to perform orphan screening, we perform a leave-one-out cross validation over all proteins, hence considering each protein as orphan once. We report the \textit{root mean squared error} (RMSE) of the predicted affinities. In order to test for stability of the  method, we create 10 distinct draws of the dataset by randomly sampling 240 ligands per protein. The reported RMSE is an average over the 10 draws per orphan protein. 
In accordance with state-of-the-art approaches we use the linear kernel for our experiments. The models for the supervised targets are trained using SVR. A 3-fold cross-validation is used to obtain the best hyperparameters for each target with parameter ranges $\epsilon\in\{0.1,0.01,0.001\}$ and regularisation parameter $C\in\{2^{-i} : i\in\{-5,-4,\dots,4,5\}\}$. More details on the experimental setup can be found in section~\ref{sec:AppendixExSe}.

Fig.~\ref{fig:rmse_linear_reduced} shows the RMSE of all approaches averaged over all orphan proteins and all draws. CP achieves a median RMSE of $2.197$. We compare CP with trivial and state-of-the-art baselines:

A naive way of combining target models without considering similarities of targets is to build a simple average over them and use the averaged model (Avg) to predict the ligand affinities for the orphan target. However, we find that Avg performs significantly worse than CP with a median RMSE of $3.610$. 

The most straight-forward way of using protein similarities is to choose models of other targets according to their similarity to the orphan target. To understand the range of how proximity of proteins affects the prediction quality we evaluate the performance of models from the targets closest to and farthest from the orphan. 
The model of the most similar target (Closest Protein) outperforms Avg, while its median RMSE of $2.427$ is significantly higher than the median RMSE of CP. Note that using the model of the least similar protein (Farthest Protein) yields a median RMSE of $3.203$, worse than CP but still better than Avg. 

The fact that Closest Protein performs better supports the intuition that the orphan and targets closer to it share similar traits which determine the affinities of ligands. Reducing Avg to a model using just the $3$ closest targets (Avg-Clo-3) yields a significant performance boost, indicating that it benefits from the focus on closer proteins. With a median RMSE of $2.260$ it outperforms Closest Protein.
\begin{wrapfigure}{r}{0.4\textwidth}
	\centering
	\includegraphics[width=5.5cm]{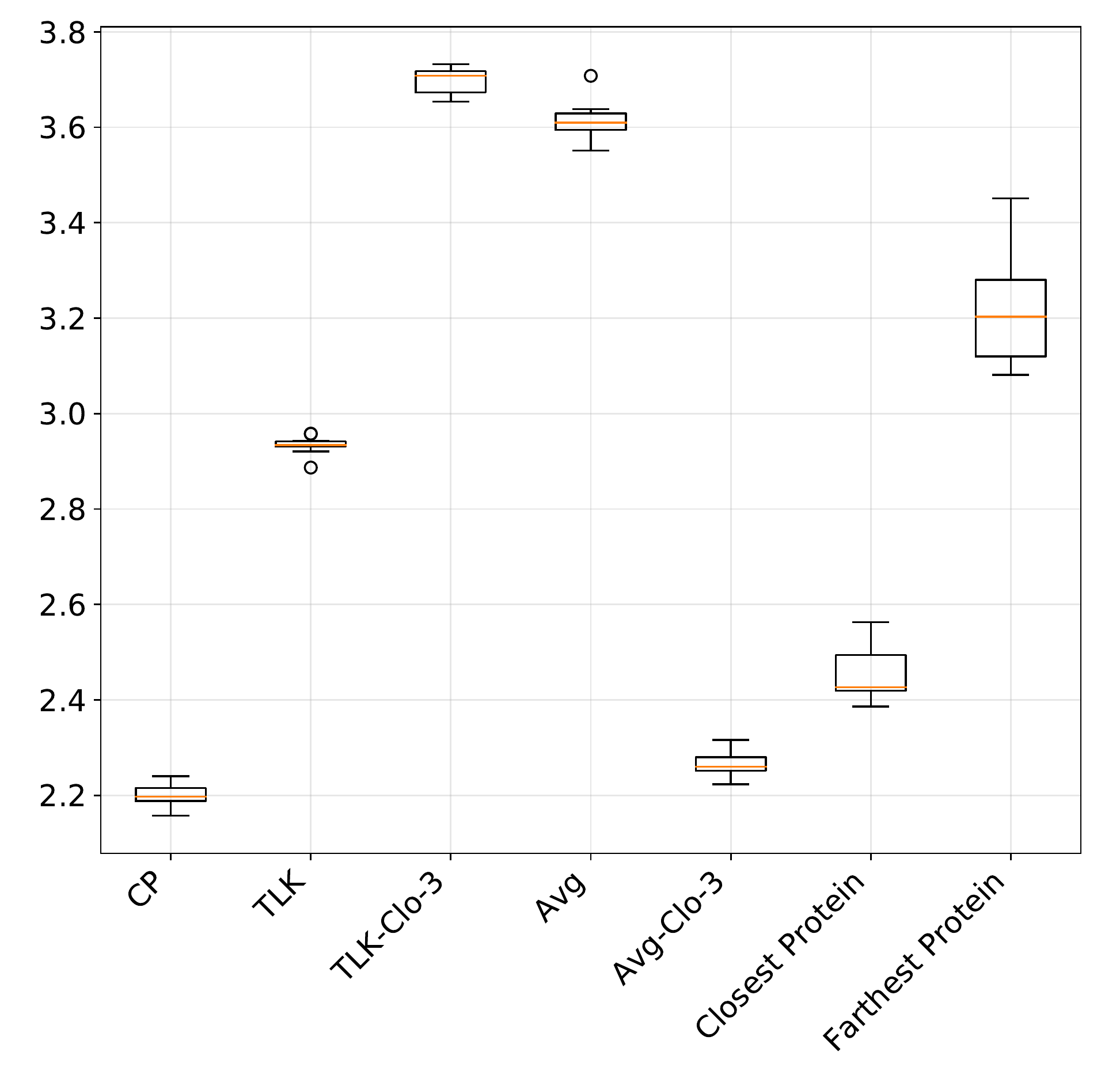}
	\caption{RMSEs of the proposed CP approach for all $9$ proteins and $10$ draws in comparison to the baselines.}
	\label{fig:rmse_linear_reduced}
\end{wrapfigure}
\begin{wrapfigure}{r}{0.4\textwidth}
	\centering
	\includegraphics[width=4cm]{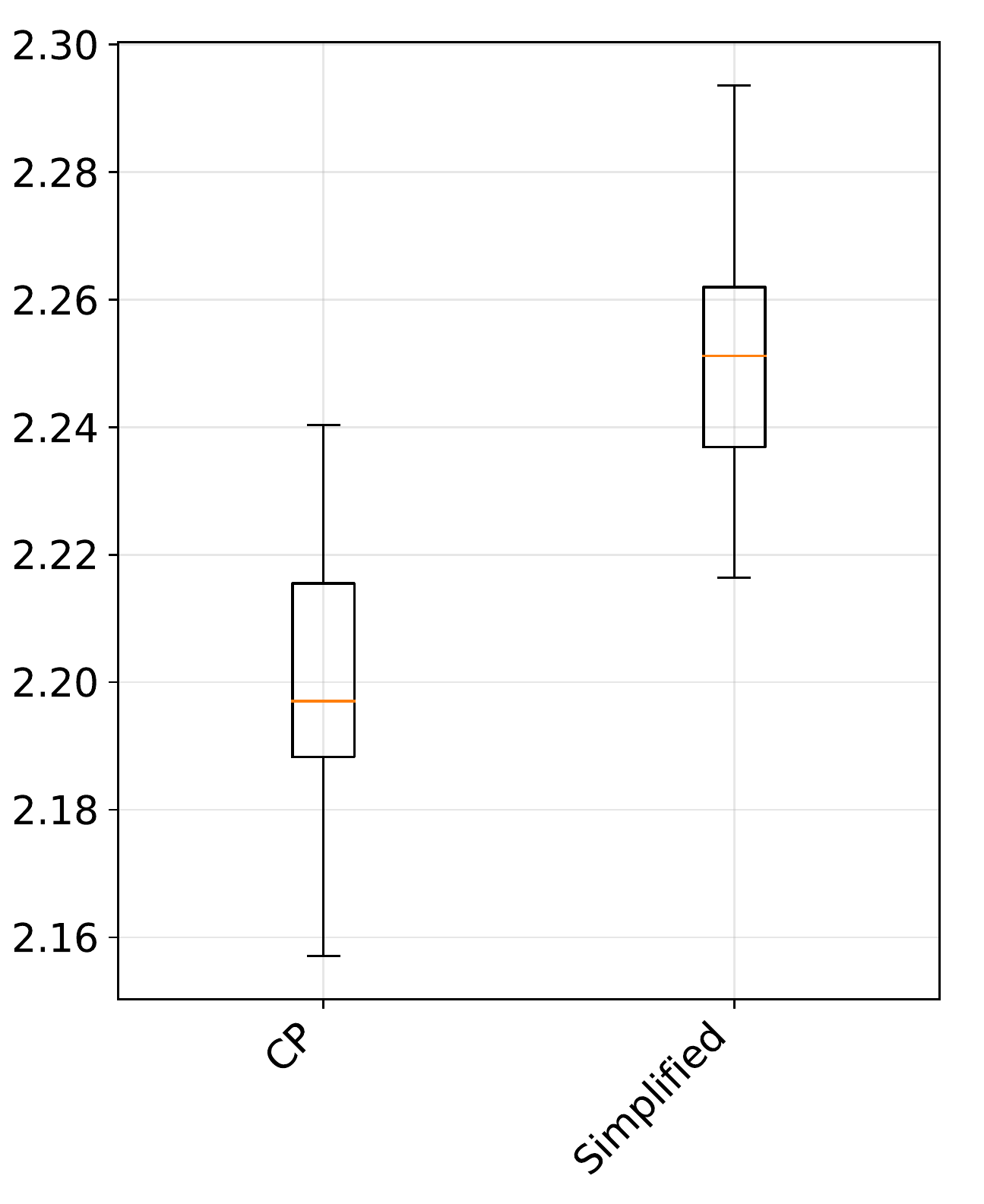}
	\caption{The RMSE of Simplified and CP averaged over all proteins and draws.}
	\label{fig:simplified_cp}
\vspace{-30px}
\end{wrapfigure}
Rather than completely omitting the targets which are not among the $3$ closest, the state-of-the-art approach TLK incorporates target similarities via features which are composed of joint target and ligand kernels. 
The hyperparameters for TLK are again optimised via grid search and 3-fold crossvalidation on all supervised targets using the same parameter ranges as above. TLK achieves a median RMSE of $2.934$, placing it between Closest and Farthest. The TLK-variant (TLK-Clo-3) which also uses just the 3 closest proteins suffers from reduced training data size and achieves a median RMSE of $3.708$, worse than the original TLK and even worse than Avg. 

Other than the previous approaches, CP optimizss the weights of each model to resemble the similarities of the targets, thereby making use of all available models. The results suggest that incorporating the similarities in the target and hypothesis space is beneficial: Not only is CP's median RMSE lower that that of simpler schemes like Avg-Clo-3, it also outperforms the state-of-the-art approach TLK.

The approach closest to CP is a simplified weighted average of models (for further details see Def.~\ref{def_scp} in Sec.~\ref{app:theory} of the appendix). Simplified outperforms the other baseline approaches, but Fig.~\ref{fig:simplified_cp} shows that it performs worse than CP, indicating that the optimisation step of CP yields a better model than the plain integration of the similarities.

\section{Conclusion and Future Work}
\label{sec:conclusion}
We introduced corresponding projections, a transfer learning based approach to hypothesis finding for unlabelled target domains that incorporates target and hypothesis similarities to derive a model for an orphan target. We applied the algorithm to the problem of target-ligand affinity prediction and empirically demonstrated its superior performance over the current state-of-the-art approach of using support vector regression with target-ligand-kernels.

For future work we will expand our evaluation of corresponding projections to further tasks, such as domain adaptation in natural language processing. We will also investigate how the proximity of targets affects the quality of predictions.

\subsubsection*{Acknowledgments}

We want to thank Stefan R\"uping and Stefan Wrobel for their valuable input and fruitful discussions. This research has been funded by the German Federal Ministry of Education and Research, Foerderkennzeichen 01S18038B.

\newpage
\bibliographystyle{plainnat}
\bibliography{bibliography}

\newpage
\newpage
\appendix
\label{sec:appendix}

\section{Appendix}

\subsection{Corresponding Projections Theory}
\label{app:theory}

In this theoretical part of the appendix we present variants of the CP optimisation as defined in Def. \ref{def_cp} as well as their corresponding solutions. We start with a linear version of corresponding projections, then we derive a simplified version and lastly we introduce the more general non-linear corresponding projection which we evaluated in this paper.

\textbf{Linear and Simplified Algorithm}

At first we turn towards the case $\mathcal{H} = \mathbb{R}^d$ with the canonical inner product, i.e., we consider linear functions of $d$-dimensional input vectors. We denote this CP version \textit{linear corresponding projections} (LCP). As a preparation for the following result we define the matrices $H \in \mathbb{R}^{d \times n}$ and $N \in \mathbb{R}^{n \times n}$ via 
\begin{align}
H = (h_1 | \cdots | h_n) ,\quad N = \operatorname{diag}(k_{\mathcal{T}}(t_i,t_i)), \label{def_mat_H_N}
\end{align}
as well as the vectors $\rho_o, \delta_o \in \mathbb{R}^n$ with
\begin{align}
\{\delta_o\}_i= k_{\mathcal{T}}(t_o,t_i) \|h_i\|_d,\quad \{\rho_o\}_i = \sqrt{k_{\mathcal{T}}(t_i,t_i)}\{\delta_{o}\}_i,\label{def_mat_rho_delta}
\end{align}
where $\|\cdot\|_d$ is the Euclidean norm in $\mathbb{R}^d$.
\begin{lm}
Let $(t_1,h_1), \ldots, (t_n, h_n) \in \mathcal{T} \times \mathcal{H}$ be examples of targets and corresponding hypotheses. If $\mathcal{H} = \mathbb{R}^d$ and $k_{\mathcal{T}}$ a similarity measure, LCP can be solved as follows
\begin{align}
f(t_o) &= \left[ \nu \mathbf{I}_d + \sum\limits_{i=1}^n h_i k_{\mathcal{T}}(t_i,t_i)h_i^T \right]^{\dagger}\label{solution_lcp}\\ 
&\quad\quad\quad\cdot\sum\limits_{i=1}^n h_i \|h_i\|_d \sqrt{k_{\mathcal{T}}(t_i,t_i)} k_{\mathcal{T}}(t_o,t_i),\nonumber
\end{align}
where $^{\dagger}$ is the \textit{Moore-Penrose inverse} of a square matrix. 
\end{lm}
\begin{proof}
We formulate the objective $\mathcal{Q}_o(h)$ in Def. \ref{def_cp} with $H, N, \rho_o$, and $\delta_o$ from above.
\begin{align}
\mathcal{Q}_o(h) = \nu h^Th + h^THNH^Th - 2h^TH\rho_o + {\delta_o}^T\delta_o \nonumber
\end{align}
The solution of LCP in (\ref{solution_lcp}) can be derived by setting the gradient of $\mathcal{Q}_o(h)$
\begin{align}
\frac{\partial \mathcal{Q}_o}{\partial h} = 2\nu h + 2HNH^Th - 2H\rho_o\nonumber
\end{align}
equal to zero. We obtain $h_o = [\nu \mathbf{I}_d + HNH^T]^\dagger H \rho_o$.
\end{proof}
As the matrix $HNH$ from above is positive semi-definite, the inverse $[\nu\mathbf{I}_d + HNH]^{-1}$ always exists if $\nu$ is positive. Otherwise, the more general $[\cdot]^{\dagger}$ will be applied. 

Subsequent to the linear version LCP we define a simplified CP variant.
\begin{defin}
Let $(t_1,h_1), \ldots, (t_n,h_n) \in \mathcal{T} \times \mathcal{H}$ be supervised targets and their hypotheses. For an arbitrary similarity function $k_{\mathcal{T}}$ on targets we define \textit{simplified corresponding projections} (SCP) according
\begin{align}
f(t_o) = \sum\limits_{i=1}^n h_i \frac{k_{\mathcal{T}}(t_o,t_i)}{\sqrt{k_{\mathcal{T}}(t_i,t_i)}}. \label{def_scp}
\end{align}
\end{defin}
In the na\"ive SCP approach the orphan hypothesis $h_o$ is a linear combination of supervised hypotheses $h_i$ with coefficients that have not to be learned beforehand. The coefficients in (\ref{def_scp}) are essentially the left hand side of the CP initial equation in (\ref{sims_equality}). Therefore, the complexity of SCP is only $\mathcal{O}(|T|d\kappa)$ if the cost for $k_{\mathcal{T}}$ is bounded by $\kappa$. In contrast, the complexity for the calculation of (\ref{solution_lcp}) is $\mathcal{O}(|T|d^2\kappa)$. Actually, for SCP the candidate space $\mathcal{H}$ is not necessarily equal to $\mathbb{R}^d$, but an arbitrary function space. A similar approach to SCP for classification already appeard in \cite{GepHumStuGaeBaj09}, where the authors also applied a weighted sum of predictors denoted as \textit{SVM linear combination} (SVM-LC).

\textbf{Non-Linear Corresponding Projections}

In the last section we considered a linear as well as a simplified version of CP. Now we want to exploit that $\mathcal{H}$ is a Hilbert space with general inner product $\langle \cdot,\cdot \rangle$ and corresponding norm $\|\cdot\|$. In this scenario we conclude a representation of $h_o$ as linear combination
\begin{align}
h_o = \sum\limits_{i=1}^n \beta_{oi} h_i \quad,\quad \beta_o \in \mathbb{R}^n, \label{linear_comb_ho}
\end{align}
i.e, $h_o$ lies in the $\operatorname{span}$ of the supervised hypotheses $h_1, \ldots, h_n$. This can be shown with an argumentation similar to the proof of the \textit{representer theorem} (RT) \cite{SchoeHerSmoWill01}. To this aim, let us consider the decomposition $h_o = s + g$, where $s \in span\{ h_1, \ldots, h_n\}$ and $\langle g, s' \rangle_{\mathcal{H}} = 0$ for all $s' \in span\{ h_1, \ldots, h_n\}$. Then we obtain
\begin{align}
&= \nu \|s+g\|^2 +\sum\limits_{i=1}^n [ \langle s + g, h_i\rangle \sqrt{k_{\mathcal{T}}(t_i,t_i)} \nonumber\\
&\quad\quad\quad\quad\quad\quad\quad\quad- k_{\mathcal{T}}(t_o,t_i) \|h_i\|~]^2 \nonumber\\
&\geq \nu \|s\|^2 +\sum\limits_{i=1}^n [ \langle s, h_i\rangle\sqrt{k_{\mathcal{T}}(t_i,t_i)} - k_{\mathcal{T}}(t_o,t_i) \|h_i\|~]^2, \nonumber
\end{align}
which shows the claim. Analogous to (\ref{def_mat_H_N}) and (\ref{def_mat_rho_delta}), we consider matrices $G, N \in \mathbb{R}^{n \times n}$ with general inner product
\begin{align}
\{G\}_{i,j} = \langle h_i, h_j \rangle_{\mathcal{H}}, \quad  N = \operatorname{diag}(k_{\mathcal{T}}(t_i,t_i))\label{def_mat_H_N_2}
\end{align} 
and vectors $\rho_o, \delta_o \in \mathbb{R}^n$
\begin{align}
\{\delta_o\}_i= k_{\mathcal{T}}(t_o,t_i)\|h_i\|_{\mathcal{H}},\quad \{\rho_o\}_i = \sqrt{k_{\mathcal{T}}(t_i,t_i)} ~\{\delta_{o}\}_i. \label{def_mat_rho_delta_2}
\end{align}
We identify $h_o$ with its defining vector $\beta_o \in \mathbb{R}^n$.
\begin{lm}
Let $\mathcal{H}$ be a Hilbert space and $k_{\mathcal{T}}$ a similarity measure on the target space $\mathcal{T}$. With the representation of $h_o$ in (\ref{linear_comb_ho}) CP from Def. \ref{def_cp} can be solved as
\begin{align}
\beta_o = [\nu G + GNG]^\dagger G \rho_o\enspace ,
\label{solution_cp1}
\end{align}
where $N$, $G$, and $\rho_o$ are defined as in (\ref{def_mat_H_N_2}) and (\ref{def_mat_rho_delta_2}). 
\end{lm}

\begin{proof}
With (\ref{linear_comb_ho}), (\ref{def_mat_H_N_2}), and (\ref{def_mat_rho_delta_2}) the objective in Def. \ref{def_cp} can be written
\begin{align}
 \mathcal{Q}_o(\beta) = \nu \beta^T G\beta + \beta^TGNG \beta -2\beta^T G \rho_o + (\delta_o)^T \delta_o .\nonumber 
\end{align}
Its gradient $\partial \mathcal{Q}_o/\partial \beta$ set to zero shows $\beta_o = [\nu G + GNG]^\dagger G \rho_o$.
\end{proof}
This constitutes the approach used throughout the paper. Solving Eq.~\ref{solution_cp1} requires inverting an $n \times n$ matrix and thus has a computational complexity of $\mathcal{O}(n^3)$. 

Note that in practice it can happen that the the union of ligands for all supervised targets $q$ is smaller than the number supervised targets $n$. This can happen, e.g., if every supervised target has the same small training set of ligands and $n$ is larger than this training set. For this case, the CP solution according to Eq.~\ref{solution_cp1} can be rewritten such that it can be solved in time $\mathcal{O}(q^3)$. 
For that we require that $\mathcal{H}$ is a \textit{reproducing kernel Hilbert space} (RKHS)
\begin{align}
\mathcal{H} = \left\{h(\cdot) = \sum\limits_{i=1}^\infty \pi_i k_{\mathcal{H}} (x_i,\cdot) : x_i \in \mathcal{X}, \alpha_i \in \mathbb{R} \right\}, \nonumber
\end{align}
where $k_{\mathcal{H}}$ defined on $\mathcal{X} \times \mathcal{X}$ is the reproducing kernel of $\mathcal{H}$ (for more details see, e.g., \cite{ShaweTayChris04, SchoeHerSmoWill01}). Assume now that each hypothesis $h_i$ arised from a training process with training examples from $\mathcal{X} \times \mathcal{Y}$ solving a regularised cost function like the one applied for \textit{regularised empirical risk minimisation}. Actually, the set of training instances and according labels depends on the respective \textit{supervised target} $t_i$. Let $\{x_1,\ldots,x_q \}$ be the union of the training instances for all targets $t_i$, $i=1,\ldots,n$, and $K$ the Gram matrix of kernel $k_{\mathcal{H}}$. The parameterised representation of each hypothesis $h_i$
\begin{align}
h_i(x) = \sum\limits_{j=1}^q \pi_{ij} k_{\mathcal{H}}(x_j,x) \quad,\quad x \in \mathcal{X}, \pi_i \in \mathbb{R}^q,\label{repr_hi}
\end{align}
exists according to the RT. If $x_j$ was not in the set of the original training instances of $t_i$ the parameter $\pi_{ij}$ is just equal to zero. We cannot apply the RT for the orphan target $t_o$ and its hypothesis $h_o$ directly because of the lack of training examples. However, $h_o$ can be represented equivalently with coefficients $\pi_o \in \mathbb{R}^q$ as we have the representation in (\ref{linear_comb_ho})
\begin{align}
h_o(x) &= \sum\limits_{i=1}^n \beta_{oi} h_i(x) = \sum\limits_{i=1}^n \beta_{oi} \left(\sum\limits_{j=1}^q \pi_{ij} k_{\mathcal{H}}(x_j,x)\right) \label{repr_ho}\\
&= \sum\limits_{j=1}^q \left(\sum\limits_{i=1}^n \beta_{oi} \pi_{ij}\right) k_{\mathcal{H}}(x_j,x) = \sum\limits_{j=1}^q \pi_{oj} k_{\mathcal{H}}(x_j,x),\nonumber
\end{align}
where $\beta_o \in \mathbb{R}^n$ and $\pi_i,\pi_o \in \mathbb{R}^q$. Hence, with $\Pi = (\pi_1 | \cdots | \pi_n)$ the coefficients of the orphan target are $\pi_o = \Pi \beta_o$. Analogous to the CP solution in (\ref{solution_cp1}) we define the matrices $\tilde{G} \in \mathbb{R}^{q \times n}$ and $N \in \mathbb{R}^{n \times n}$
\begin{align}
\tilde{G} = K\Pi,\quad N = \operatorname{diag}(k_{\mathcal{T}}(t_i,t_i)) \label{def_mat_H_N_3}
\end{align}
and vectors $\tilde{\rho}_o, \tilde{\delta}_o \in \mathbb{R}^n$ 
\begin{align}
\{\tilde{\delta_o}\}_i = k_{\mathcal{T}}(t_o,t_i) \sqrt{\pi_i^T K\pi_i},\quad \{\tilde{\rho}_o\}_i = \sqrt{k_{\mathcal{T}}(t_i,t_i)} ~\{\tilde{\delta}_{o}\}_i\label{def_mat_rho_delta_3}
\end{align}
Again, we identify $h_o$ with its vector of coefficients $\pi_o$. 
\begin{lm}
Let $h_i$, $i = 1,\ldots,n$, and $h_o$ have the representations (\ref{repr_hi}) and (\ref{repr_ho}) from above with $\pi_i, \pi_o \in \mathbb{R}^{q}$. With $k_{\mathcal{H}}$ and $k_{\mathcal{T}}$ we denote the reproducing kernel of $\mathcal{H}$ and a similarity measure on the target space $\mathcal{T}$, as well as $K$ be the Gram matrix with respect to $x_1,\ldots,x_q$. With (\ref{repr_hi}) and (\ref{repr_ho}) CP can be solved as
\begin{align}
\pi_o &= \left[ \nu K + \sum\limits_{i=1}^n K \pi_i k_{\mathcal{T}}(t_i,t_i) \pi_i^T K \right]^{\dagger} \nonumber\\ 
&\quad\quad\cdot\sum\limits_{i=1}^n \left( \sqrt{\pi_i^T K \pi_i} \sqrt{k_{\mathcal{T}}(t_i,t_i)} k_{\mathcal{T}}(t_o, t_i)\right) K \pi_i, \label{solution_cp2}
\end{align}
where $\nu \geq 0$ and $t_i$, $i=1,\ldots,n$, are the supervised targets. We call this approach kernel corresponding projections (KCP).
\end{lm}

\begin{proof}
The proof is again a consequence of the stationarity of the gradient of the parameterised objective
\begin{align}
\mathcal{Q}_o(\pi) = \nu \pi^T K \pi + \pi^T \tilde{G} N \tilde{G}^T \pi - 2 \pi^T  \tilde{G} \tilde{\rho_o} + \tilde{\delta_o}^T \tilde{\delta_o}\nonumber
\end{align}
from Def. (\ref{def_cp}).
\end{proof}
Solving Eq.~\ref{solution_cp2} requires inverting a $q\times q$ matrix and thus has a computational complexity in $\mathcal{O}(q^3)$. As mentioned above, this is preferable to solving Eq.~\ref{solution_cp1} for $n\gg q$.

\subsection{Extended Experimental Results}
\label{sec:AppendixExSe}
In this section we provide information on the results with supervised baselines.

We also evaluated a hypothetical supervised case in which we trained an SVR on several fractions ($5\%,10\%,30\%,50\%,$ and $80\%$) of the available ligands, testing it on remaining ones. 
\begin{table}[]
\centering
\begin{tabular}{c|c}
\textbf{Method} & \textbf{Median RMSE} \\ \hline
\\
CP              & 2.197                \\
Supervised-5\%  & 1.797                \\
Supervised-10\% & 1.440                \\
Supervised-30\% & 1.140                \\
Supervised-50\% & 1.038                \\
Supervised-80\% & 0.946               
\end{tabular}
\caption{Median RMSE of CP and the supervised approaches over all draws, averaged over all proteins.}
\end{table}
This yields an upper bound for the RMSE. On average all supervised models clearly outperform CP. However, they solve a different task, since CP assumes no labelled data for the orphan target. Notably, the performance of CP is a lot closer to the supervised approaches, than TLK. For Supervised 50\% the median RMSE of CP is larger than Supervised 50\% by a factor of $2$, while TLKs median RMSE is larger by a factor of almost $3$. 
\begin{figure}
	\centering
	\subfigure[]{\includegraphics[width=6.5cm]{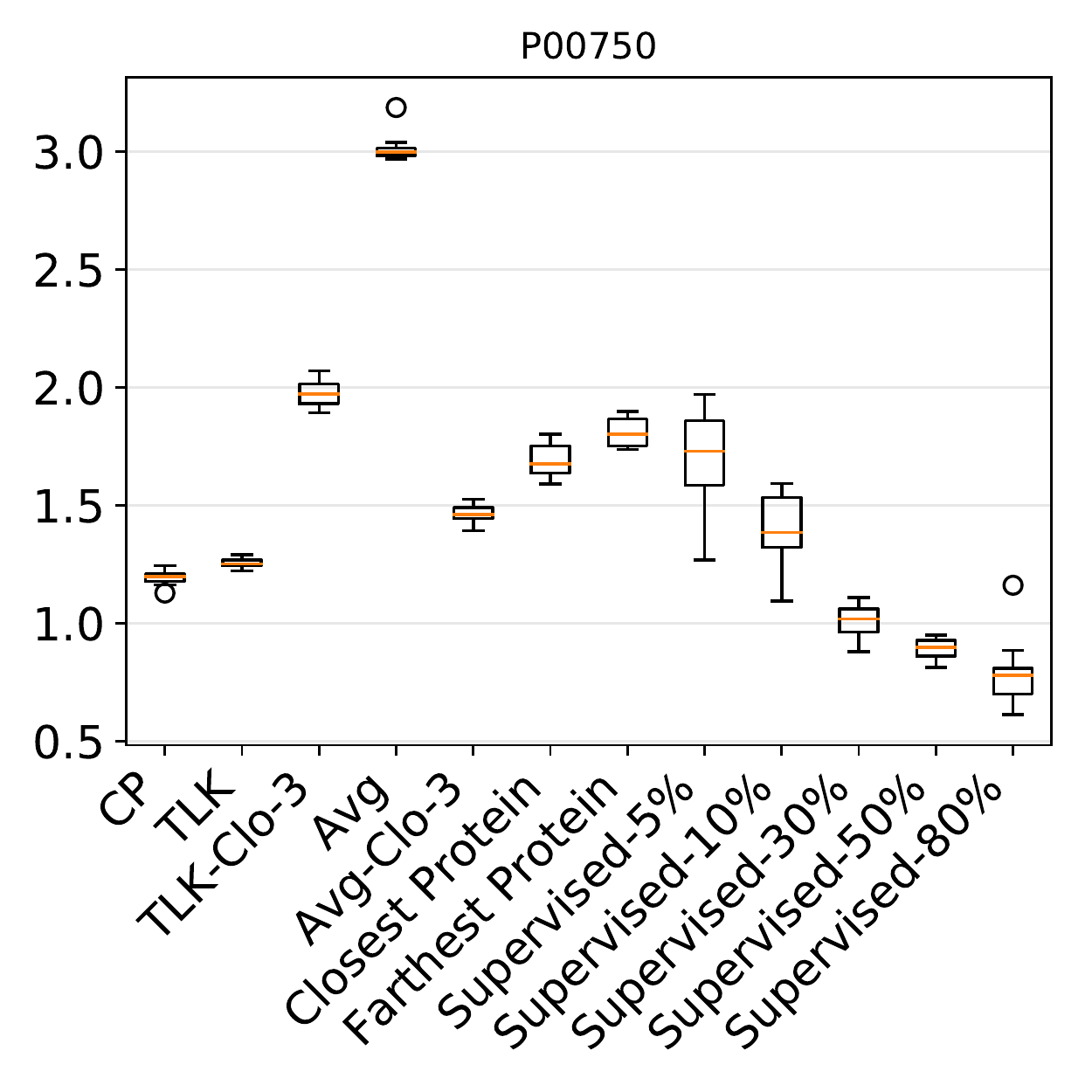}\label{fig:P00750}}
	\subfigure[]{\includegraphics[width=6.5cm]{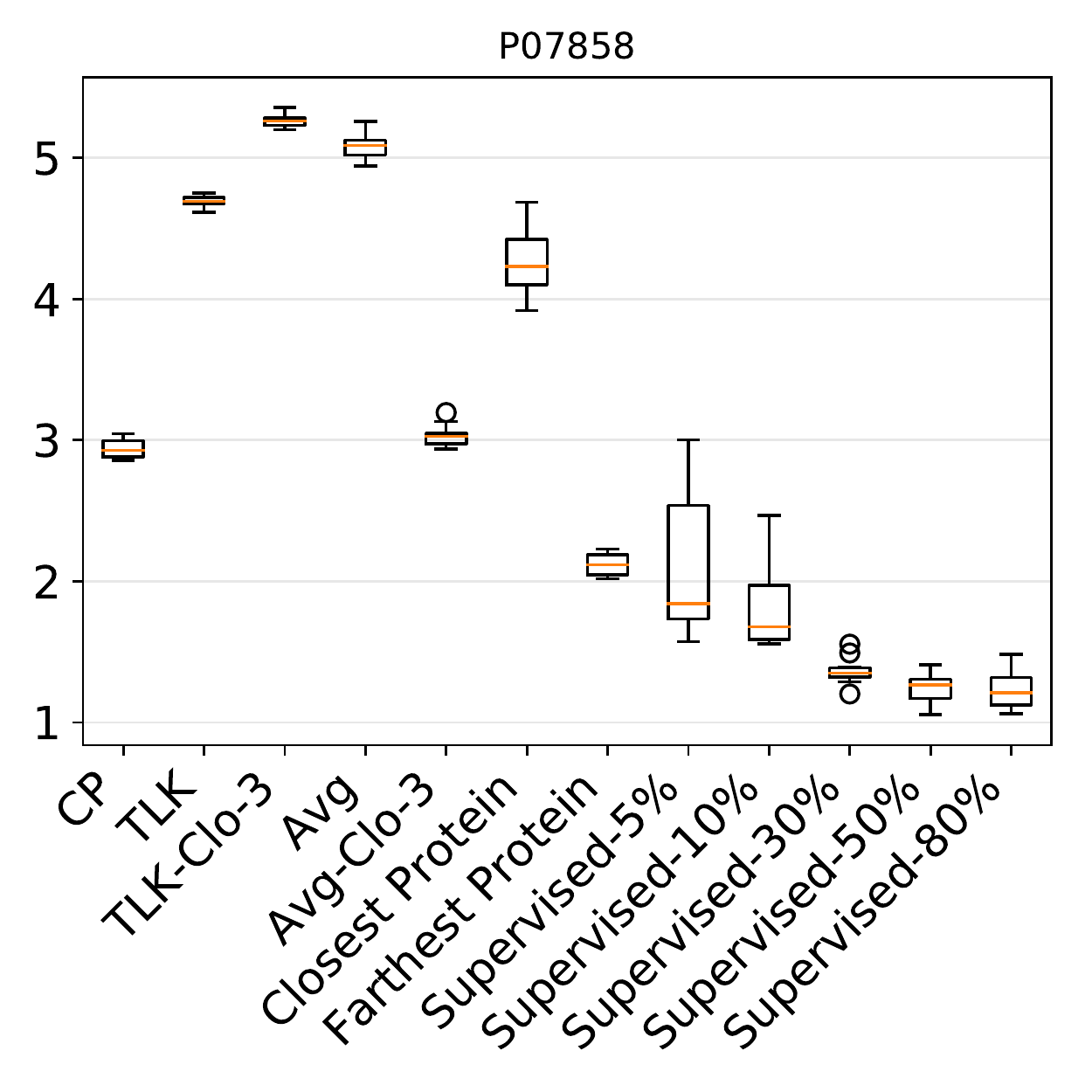}\label{fig:P07858}}
	\caption{The RMSE of the unsupervised and supervised approaches over all draws for two proteins protein \textit{P00750} and \textit{P07858}.}
	\label{fig:rmsePerProtein}
\end{figure}
A more detailed investigation of the results, shows that CPs performance varies strongly depending on the orphan target. Fig.~\ref{fig:rmsePerProtein}\subref{fig:P00750} shows the performance for the protein named \textit{P00750}. We can observe that CP outperforms both the supervised model trained on $5\%$ and on $10\%$ of the data, reaching a performance close to a supervised model trained on $30\%$ of the data. Contrary Fig.~\ref{fig:rmsePerProtein}\subref{fig:P07858} shows that for protein \textit{P07858} CP is worse than all supervised approaches and the model of the farthest protein performs best out of all unsupervised approaches, nearly as good as the supervised model trained on $10\%$ of the data. In total CP performs better than Supervised-5\% for 4 proteins, better than Supervised-10\% for 2 proteins, comparable to Supervised-5\% for 1 protein and significantly worse than all supervised approaches for 4 proteins. 

This raises open research questions left for future work. One question is whether the chosen similarity measure for target similarities is suitable or whether other similarity measures would perform better. Another one is whether the performance of CP increases with the number of ligands per protein or the total number of proteins. Since the performance of CP varies between proteins, another question is whether this performance can be related to the similarity of the orphan target and supervised ones.
%
%

\end{document}